\documentclass{article}
\usepackage{enumitem}
\usepackage{wrapfig}
\usepackage[preprint]{corl_2026} %
\usepackage{times}
\usepackage{graphicx}
\usepackage{amsthm}
\usepackage[table,xcdraw,dvipsnames]{xcolor}
\usepackage{booktabs}
\usepackage[numbers]{natbib}
\usepackage{multicol}

\newcommand{\Sa}{\ensuremath{\mathcal{S}}}

\newcommand{\Aa}{\ensuremath{\mathcal{A}}}

\usepackage{amsmath,amsfonts,bm}

\def\Figref#1{Figure~\ref{#1}}

\def\Secref#1{Section~\ref{#1}}

\def\eqref#1{equation~\ref{#1}}
\def\Eqref#1{Equation~\ref{#1}}

\def\1{\bm{1}}

\DeclareMathAlphabet{\mathsfit}{\encodingdefault}{\sfdefault}{m}{sl}
\SetMathAlphabet{\mathsfit}{bold}{\encodingdefault}{\sfdefault}{bx}{n}

\DeclareMathOperator*{\argmax}{arg\,max}
\DeclareMathOperator*{\argmin}{arg\,min}

\newcommand{\methodname}{\emph{Deliberate Practice}}

\newtheorem{theorem}{Theorem}

\newtheorem{lemma}{Lemma}
\title{Deliberate Practice:\\ Learning Robot Skills under a Budget}

\author{
   Shivam Vats$^{1}$  \quad
   Sudarshan Harithas$^{1}$ \quad
   Mete Tuluhan Akbulut$^{1}$ \\
   \textbf{Arvind Raghunathan}$^{2}$ \quad
   \textbf{George Konidaris}$^{1}$ \\
  $^{1}$Brown University \quad
  $^{2}$Mitsubishi Electric Research Laboratories \\
  Correspondence to \texttt{shivam\_vats@brown.edu}\\
}

\begin{document}
\maketitle

\begin{abstract}
We consider the problem of autonomously learning robot skills under a limited practice budget for sequential tasks. We propose an active skill learning algorithm, \emph{Deliberate Practice (DP)}, that computes a provably \emph{budget-optimal} allocation---practicing skills that maximize expected cumulative reward while being learnable within the budget. DP estimates both the time needed to master skills and the cumulative reward of the task plans that the skills unlock. Computing a budget-optimal allocation is challenging as it requires reasoning about combinatorially many skill plans over a large practice budget. Our key contribution is a bilinear program that can compute this exactly using off-the-shelf solvers. Through simulated and real-world experiments on long-horizon manipulation tasks, we show that our approach allows robots to optimally use limited practice time to acquire useful policies and improve long-horizon planning.
\end{abstract}


\section{Introduction}

Recent progress in large-scale robot learning has produced general-purpose robot policies that can perform diverse skills and transfer to novel environments~\cite{pastor2011skill,zhu2023learning,chi2025diffusion}.
When combined with classical planning methods, such as Task and Motion Planning (TAMP)~\cite{cambon2009hybrid,garrett2021integrated,hedegaard2025beyond} and search~\cite{liang2022search,mishani2025mosaic}, these policies can allow robots to perform long-horizon tasks in homes and factories, charting a path for more autonomous robots.
However, pretraining alone does not cover all the possible scenarios a robot may face in the real-world, often leading to unreliable task execution.
Therefore, a promising solution is to allow robots to learn from experience through autonomous practice~\cite{kober2013reinforcement,deisenroth2013survey,gu2017deep}.
Unfortunately, current RL algorithms remain highly sample inefficient~\cite{mnih2015human,kalashnikov2018scalable}, limiting their applicability to deployment-time learning. Since robots typically have only limited downtime during deployment, they require methods explicitly designed for budgeted practice.

Our key insight is that downtime is often known in advance, and robots should leverage this information to adapt their learning process to the resulting {practice budget}. Intuitively, a large practice budget should encourage robots to practice more rewarding skills even if they are hard, while a small budget should lead to conservative learning behaviour.
We formalize the problem of autonomous skill practice for sequential tasks as budgeted skill learning, and propose an active learning algorithm, \methodname{} (DP), that optimally uses the budget to maximize task reward.
We consider the standard TAMP setting in which a robot is given high-level skill specifications, including preconditions, termination conditions, and effects. For example, in the  breakfast domain in \Secref{sec:experiments}, the \texttt{StartToaster} skill requires moving above the toaster lever to press it and  has the effect of turning on the toaster.
The robot must learn parameterized control policies that ground skills into low-level actions through practice in the environment.
At decision time, a task planner computes a plan consisting of skills, e.g., \texttt{(Pick(bread), Place(bread, toaster), Start(toaster))} to toast bread.
Figure~\ref{fig:teaser} illustrates how our approach allows the robot to decide which skills to practice to make breakfast. The robot can either practice one skill to toast bread or use additional practice time to learn two skills required to microwave oatmeal and achieve a higher reward.
As shown in \Secref{sec:experiments}, with a budget of $30$ episodes, DP correctly estimates that only one skill can be learned reliably and hence practices to toast bread.
Under a budget of $60$ episodes, it instead practices to microwave oatmeal to take advantage of the additional practice time.

\begin{figure}[t]
    \centering
    \includegraphics[width=1\linewidth]{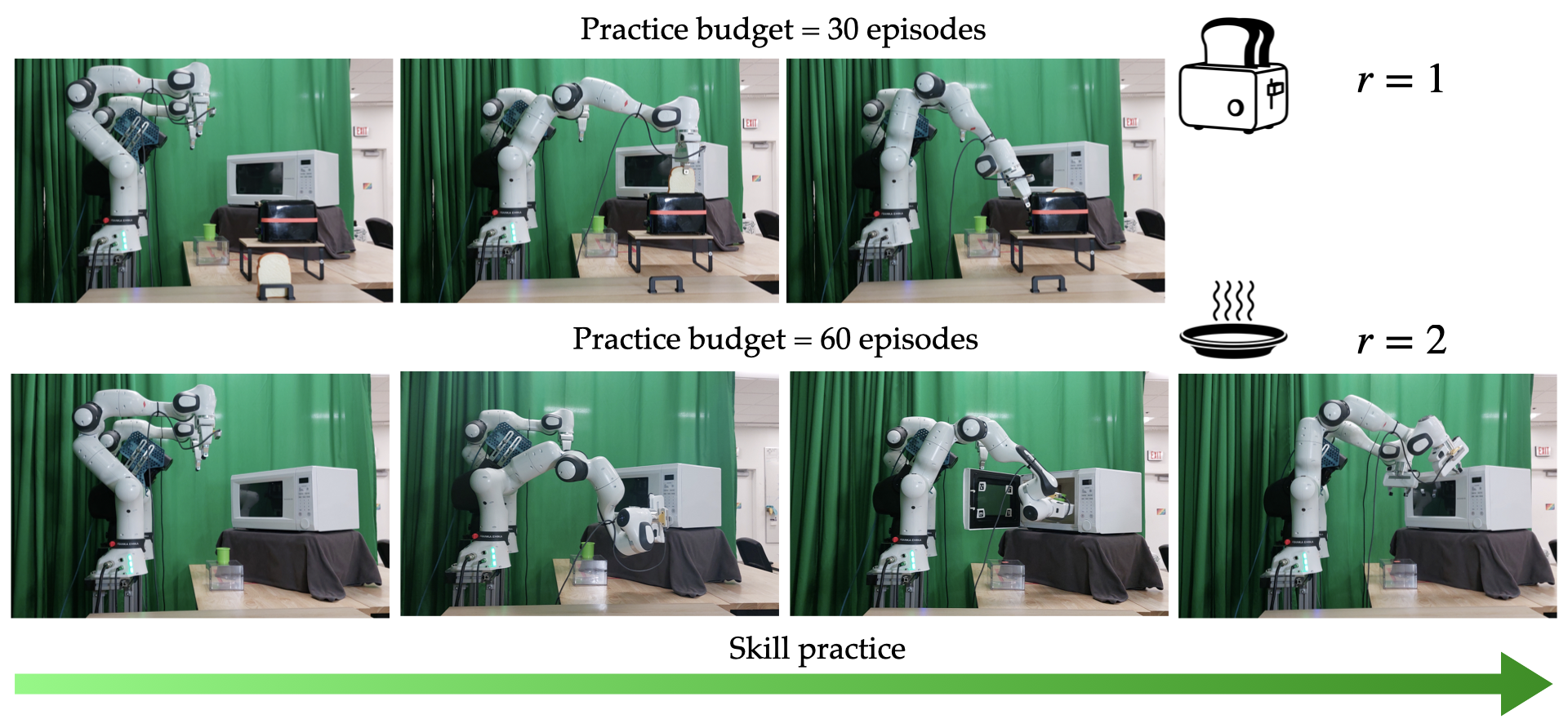}
    \caption{\methodname{} allows robots to intelligently use a limited practice budget to actively learn skills to maximize long-horizon task reward. \emph{(top)} Under a small budget, the robot practices an easy task plan that toasts bread and provides reward $1$.
    \emph{(bottom)}  Under a larger budget, it practices a more difficult but higher-reward task plan that microwaves oatmeal and provides reward $2$.
    }
    \label{fig:teaser}
\end{figure}

\methodname{} first estimates the difficulty of learning each skill as a function of the practice budget.
It then computes a {budget allocation} that provably maximizes expected planning performance, and practices its skills accordingly.
We show in \Secref{sec:formulation} that computing an optimal allocation requires solving a challenging bilevel optimization problem.
Existing active learning algorithms approximate this problem greedily, leading to myopic and suboptimal learning~(\Secref{sec:greedy}).
By contrast, we derive an exact single-level reformulation of the bilevel problem  using the linear programming (LP) formulation of Markov decision processes. This allows us to compute the optimal budget allocation using off-the-shelf solvers.
Through experiments in simulated and real-world long-horizon manipulation tasks, we demonstrate that \methodname{} is uniquely budget-aware and enables robots to learn effectively under deployment constraints.

\section{Related Work}

\textbf{Active Skill Learning.}
Deliberate practice is a widely studied mechanism underlying expert human performance~\cite{anders2008deliberate}, characterized by structured, goal-directed training that allocates effort toward the most limiting components of a skill.
In robotics, active learning has similarly been central for skill acquisition due to the high cost of data collection. Prior work includes methods for actively learning parameterized robot skills from practice~\citep{da2014active}, human demonstrations~\citep{chernova2009interactive}, and across multiple tasks~\citep{fabisch2014active,vats2022synergistic,vats2025coil}.
However, these approaches focus on efficiently acquiring individual skills without considering the sequential nature of long-horizon problems.
By contrast, our approach studies {active skill learning embedded in a complex sequential decision process under explicit budget constraints}.

\textbf{Task and Motion Planning (TAMP).}
TAMP~\citep{cambon2009hybrid,srivastava2014combined,garrett2021integrated} addresses robotic planning problems that require joint reasoning over discrete task structure and continuous motion feasibility.
Most TAMP systems adopt a hierarchical approach, where a high-level planner~\citep{fox2003pddl2,hauser2010task} solves an abstract version of the problem to compute a task plan, which is then grounded into low-level robot actions using motion planning and control~\citep{garrett2020pddlstream}.
Recent TAMP approaches leverage machine learning to improve scalability and performance, for example, by learning heuristics to guide symbolic search \citep{chitnis2016guided}, or by learning symbolic operators and task representations~\citep{silver2021learning,shah2024hierarchical}.
Our work is closely related to approaches that \emph{plan to practice} skills for autonomous improvement in long-horizon tasks~\citep{vats2023efficient,kumar2024practice}. However, these methods rely on greedy practice strategies and often allocate practice time suboptimally. By contrast, our method computes a budget-optimal allocation of practice time.

\section{Background}
We model the environment as a Markov decision process (MDP) with continuous state and action spaces, defined by the tuple $\mathcal{M} :=  (\mathcal{S}, \mathcal{A}, P, r, \gamma)$, where $\mathcal{S}$ is the low-level state space, $\mathcal{A}$ is the low-level robot action space, $P(s'|s,a)$ is the transition function, $r(s,a)$ is the reward function, $\gamma$ is the discount factor.
We assume an abstract MDP $\mathcal{M}_h :=  (\mathcal{S}_h, \mathcal{A}_h, P_h, r_h, \gamma)$, where a state abstraction function typically $\alpha: \mathcal{S} \rightarrow \mathcal{S}_h$ projects raw sensory observations to abstract states.

\textbf{Abstract Dynamics.} We assume that the robot is provided with a library of skills $\mathcal{U}$, which are defined using an extension of the options framework, consisting of the following components:
a) \emph{Precondition} (\textsc{Pre}) defines the states in which a skill can be initiated,
b)  \emph{Effect} (\textsc{Eff}) describes the outcome of executing a skill successfully,
c)  \emph{Control policy} ($\pi$) computes the low-level actions to be  executed by the robot,
d) \emph{Termination condition} ($\beta$) specifies when a skill should terminate,
e) {Competence} ($p$) is the probability that the skill will successfully achieve the desired effect.
The abstract action space $\mathcal{A}_h$ is defined by skill library $\mathcal{U}$ that the robot can use to transition between abstract states.
Skill effect and competence together define the high-level transition function $P_h$: $P_h(s'|s,a) = p$, if $s'$ corresponds to the effect of skill $u \in \mathcal{U}$ when executed in state $s$, otherwise transitioning to an absorbing failure state.
We use a task planner to solve this abstract MDP $\mathcal{M}_h$ to compute a sequence of skills with the correct parameters to reach the goal, e.g., \texttt{(Open(top-drawer), Pick(toy), Place(toy, top-drawer))}.
The skills then ground this abstract task plan by computing feasible robot trajectories.

\textbf{Structured Robot Skill.} We use a combination of motion planning and parameterized policies to handle contact-rich interactions with the environment. We structure every contact-rich policy as an object-centric Composable Interaction Primitive (CIP)~\citep{abbatematteo2024composable}, which consists of three phases:
a) \emph{Pre-interaction:} motion plan to a pose near the object, b) \emph{Interaction:} execute a parameterized policy, and c) \emph{Post-interaction:} motion plan away from the object.
The interaction policy is learned by the robot by practicing in the environment.

\textbf{Task and Skill Reward.} The robot is tasked with maximizing a task reward function $R$ with one or more absorbing goal states, defined by a goal function $g: \mathcal{S} \rightarrow \{0, 1\}$.
We also define skill-specific reward functions to learn the policy parameters through practice.

\section{Budgeted Skill Learning for Planning}
\label{sec:formulation}

Given the high-level descriptions of a library of skills $\mathcal{U}$, consisting of their preconditions, termination conditions, and effects, the robot must learn the policy parameters of its skills to improve overall planning performance under a limited practice budget $B$, e.g., maximum number of trials.
We adopt the standard episodic reinforcement learning setting without access to expert demonstrations.
Our goal is to compute a budget allocation $b_{u}$ across all skills $\mathcal{U}$ to maximize the expected planning performance. Formally, this can be expressed as the following optimization problem:
\begin{equation}
\begin{split}
    &\max_b \sum_{s\in\Sa} e_s v_s \\
    &\text{subject to} \\
    &\quad \sum_{u \in \mathcal{U}} b_{u} \leq B, \\
      &\quad v_s \in \text{SolveMDP}(\mathcal{M}_h, b),
\end{split}
\label{eq:bilevel_obj}
\end{equation}

where $e_s$ is the initial state distribution, and $v_s$ is the state value function. This is a bilevel optimization problem where the outer problem computes a budget allocation $b$ for practicing policies, and the inner problem solves the resulting task MDP $\mathcal{M}_h$ to compute the expected planning performance after practicing with budget allocation $b$.

\subsection{Greedy Active Learning Converges to Local Minima}
\label{sec:greedy}

\begin{wrapfigure}{r}{0.35\columnwidth} %
    \centering
    \includegraphics[width=\linewidth]{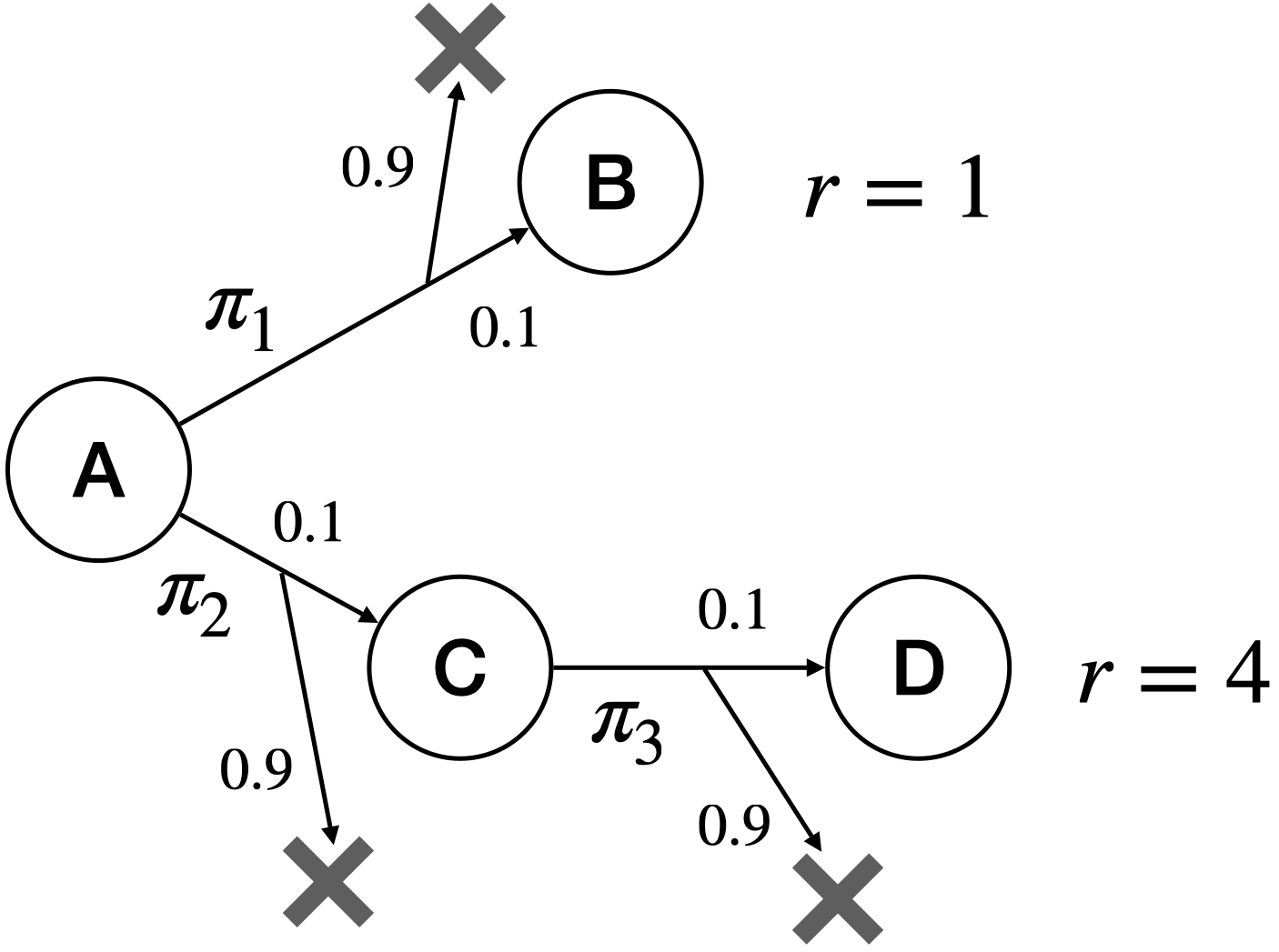}
    \caption{An example illustrating how greedy active learning can make suboptimal decisions.}
    \label{fig:greedy_counter}
\end{wrapfigure}
Prior active skill learning methods~\citep{vats2023efficient,kumar2024practice} for sequential tasks greedily practice the skill that will most improve the expected task performance $J_{\text{task}}$ after one round of training.
These methods can be viewed as performing one-step reasoning to solve \eqref{eq:bilevel_obj}.
However, this is inherently myopic: it can overlook higher-performing task plans that require deliberate practice over multiple episodes to be mastered.
We provide an illustrative example in Figure~\ref{fig:greedy_counter}.
This MDP consists of $5$ states: A, B, C, D and a terminal failure state. The robot always starts at A, while B and D are terminal goal states. The MDP can be solved either by using $\pi_1$ to reach B or by using $\pi_2$ and $\pi_3$ to reach D.
All skills have identical competence $p = 0.1$; they transition to the failure state with probability $0.9$.
The robot must decide which skills to practice among $\pi_1, \pi_2, \pi_3$ under a budget of 20 episodes given each skill has the same rate of improvement $\Delta p_i = 0.1$ per practice episode.
The optimal policy for the MDP before practice is to use $\pi_1$ to reach B. Hence, $J_{\text{task}} =  p_1 \times R_B = 0.1$.
Let $\pi_i'$ be the skill after a round of practice. The greedy task improvement measure of each skill is computed below:
\begin{align*}
& \text{Practice $\pi_1$} \implies J_{\text{task}}(\{\pi_1', \pi_2, \pi_3\})  =  (p_1 + \Delta p_1) \times R_B  = 0.2 \times 1  = 0.2 \\
& \text{Practice $\pi_2$} \implies J_{\text{task}}(\{\pi_1, \pi_2', \pi_3\}) =  (p_2 + \Delta p_2) \times p_3 \times R_D  = 0.2 \times 0.1 \times 4  = 0.08 \\
& \text{Practice $\pi_3$} \implies J_{\text{task}}(\{\pi_1, \pi_2, \pi_3'\})  =  p_2 \times (p_3 + \Delta p_3) \times R_D  = 0.1 \times 0.2 \times 4  = 0.08.
\end{align*}

Therefore, greedy methods practice only $\pi_1$ since it always maximizes one-step $J_{\text{task}}$ improvement resulting in the final reward $1$.
However, this misses the optimal solution $\{\pi_2, \pi_3\}$ which will provide reward $4$ under $20$ rounds of practice, because identifying the optimal budget allocation requires reasoning about deliberate practice across $\pi_2$ and $\pi_3$ over multiple rounds.

\section{Deliberate Practice}

Our algorithm allocates the practice budget provably optimally by jointly reasoning about all robot skills across the entire practice budget.
As shown in \Figref{fig:overview}, our approach comprises three steps:
\begin{enumerate}
    \item \textbf{Competence Prediction.} Competence improvement is modeled as a function of the practice budget. This model is initialized with a domain-specific prior and updated online based on the robot's actual improvement.
    \item \textbf{Budget Allocation.} A budget-optimal task plan $\Pi^*$ is computed along with the corresponding budget allocation $b^*$ needed to master it.
    \item \textbf{Skill Practice.} The robot sequentially masters the skills that are reachable from the start and then uses them to plan to reach and practice downstream skills.
\end{enumerate}

\begin{figure}[t]
    \centering
    \includegraphics[width=1\linewidth]{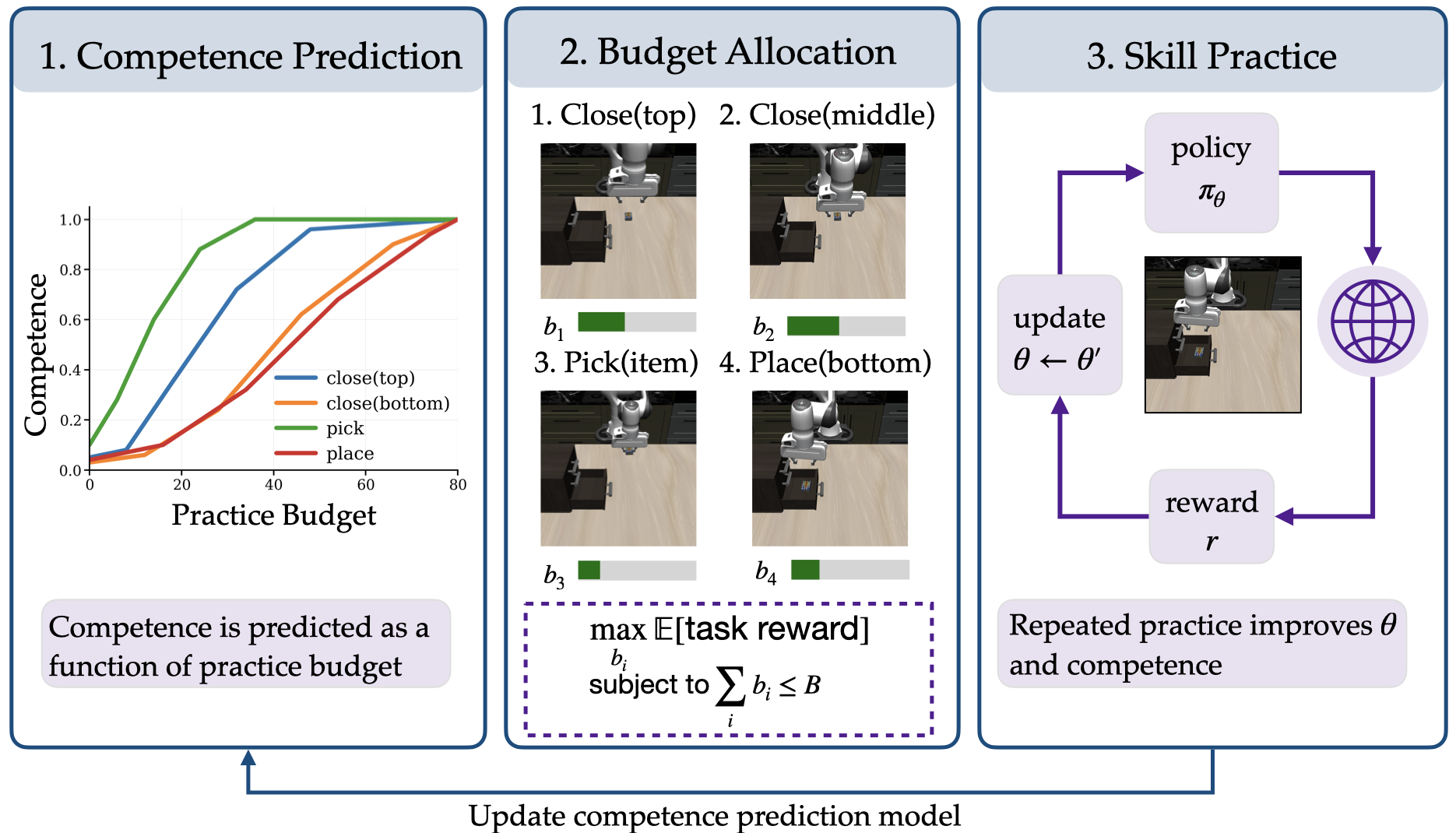}
    \caption{\textbf{Overview:} Our approach predicts skill competence as a function of practice budget, computes a budget allocation for practice, and practices skills by interacting with the environment.}
    \label{fig:overview}
\end{figure}

\subsection{Competence Prediction}
The competence prediction function $f_{\text{improv}}(p_u, b_u)$ predicts the expected skill competence after practicing a skill with a competence of $p_u$ for $b_u$ rounds.
While our approach supports linear, piece-wise linear and exponential models of competence prediction,  empirically a piece-wise linear model provides a good balance between performance and efficiency.
\begin{equation}
    f_{\text{improv}}(u, b) = \min(1, p_{u} + \Delta_{u}b), \label{eq:competence_improvement_pwl}
\end{equation}
where $\Delta_{u}$ is the rate of improvement of a skill $u$ estimated from practice. $\Delta_{u}$ is estimated online as a running average of competence improvement observed after practicing skill $u$ over time: $\Delta_u^t = \epsilon\Delta_u^{t-1} + (1 - \epsilon)(p_u^t - p_u^{t-1})$, where $\epsilon \in [0,1]$ is a smoothing factor. To capture highly non-linear dynamics (such as deep reinforcement learning), our approach also supports more complex models, such as the saturating exponential: $f_{\text{improv}}(u, b) = p_u + (1-p_u)(1 - e^{-\alpha b})$, where $\alpha$ is the learning rate that is estimated online.

\subsection{Budget Allocation}
Next, we use the competence prediction model to compute the best task plan feasible within the practice budget.
Long-horizon manipulation tasks require deliberately practicing multiple skills over many episodes to be reliably solved. As the bilevel optimization problem (\Eqref{eq:bilevel_obj}) shows, this requires jointly reasoning about all skills and the entire practice budget.
However, this optimization is notoriously difficult due to its nested structure: the feasible set of the outer problem is defined implicitly by the solution of the inner optimization. This coupling leads to highly nonconvex and often non-smooth objectives, even when both levels are individually convex.
We derive an \textbf{exact single-level reformulation} using duality, yielding a structured single-level bilinear program that enables efficient optimization using standard nonlinear programming solvers, eliminating the need for nested optimization, and greedy approximations.
To our knowledge, this is the \emph{first exact formulation of robot learning on a budget as an optimization problem}.

We first state the bilinear program and then outline the derivation. The lines in black in \eqref{eq:single_level} correspond to the standard dual linear program of an MDP~\cite{puterman2014markov}, and our proposed changes are highlighted in {\color{NavyBlue}{blue}}.
$\mu_s^a$ are dual variables that correspond to state-action occupancy, $e_s$ is the initial state distribution, $P_{s's}^a$ is the current transition function, and $\bar{P}_{s's}^a(b)$ is the predicted transition function after practicing with budget allocation $b$.
Surprisingly, we show that practice budget constraints can be directly incorporated into the standard dual linear program.
\begin{equation} 
\begin{split}
&\max_{b, \mu} \sum_{\forall s\in\Sa} \sum_{\forall a \in \Aa} r_s^a\mu_s^a \\
&\text{subject to}\\
&\quad\sum_{\forall a \in \Aa} \mu_s^a - \gamma \sum_{s'} \sum_{\forall a } \bar{P}_{s's}^a(b) \mu_{s'}^a = e_s, \forall s \\
&\quad {\color{NavyBlue}\bar{P}_{s's}^a(b) = f_{\text{improv}}({P}_{s's}^a, b) } \\
&\quad{\color{NavyBlue}\sum{b_{ij}} \leq B} \\
&\quad{\color{NavyBlue}\mu_s^a \leq \frac{1}{1 - \gamma}, \forall s, a}.
\end{split}
\label{eq:single_level}
\end{equation}

This is a \emph{nonsmooth bilinear mathematical program} because of the piece-wise linear model in~\eqref{eq:competence_improvement_pwl} and bilinear terms $\bar{P}_{s's}^a(b) \mu_{s'}^a$ in the constraints.
While being a nonsmooth, nonconvex nonlinear problem, there exist powerful techniques to solve it to global optimality. In particular, \emph{piecewise McCormick envelopes} can be used to develop convex relaxations of the bilinear constraints and are supported by popular solvers, such as Gurobi~\citep{gurobi}.

\begin{theorem}
Deliberate Practice is budget-optimal, i.e, it computes a globally optimal allocation of the practice budget.
\end{theorem}
\begin{proof}
First, we show that the bilinear program in \Eqref{eq:single_level} is an exact reformulation of the bilevel program in \Eqref{eq:bilevel_obj}.
Our key idea is to replace the inner optimization with its linear programming formulation. This allows us to convert the bilevel problem into a $\max \min$ problem. Then, we utilize LP duality to reformulate it as a single-level $\max$ problem with bilinear constraints.
We solve~\eqref{eq:single_level} to global optimality using spatial branch and bound methods~\citep{tawarmalani2013convexification}, which handle bilinear constraints via convex relaxation techniques, such as Piecewise McCormick Envelopes.
Constructing these relaxations requires valid finite bounds on all the variables appearing in the bilinear constraints.
We derive tight lower and upper bounds for all variables $\mu$ and $P$ appearing in bilinear constraints, enabling the certification of global optimality.
In particular, $P_{ss'}^a \leq1$ since it corresponds to transition probability, and $\mu_s^a \leq \frac{1}{1-\gamma}$ since it is the discounted state-action occupancy measure.
A detailed proof is in Appendix~\ref{sec:single_level_formulation}.
\end{proof}

\subsection{Skill Practice}
Once the robot computes a budget allocation and a corresponding budget-optimal task plan, it constructs a curriculum to practice the selected skills.
This requires the robot to plan with its existing skills to first reach the precondition of the skill being practiced~\citep{kumar2024practice}.
Specifically, the robot sequentially masters each of the skills on the budget-optimal plan computed by \methodname{} and uses each newly acquired skill to reach and practice subsequent skills.

\section{Experimental Evaluation}
\label{sec:experiments}

We evaluate our method in three long-horizon table-top manipulation tasks. Our simulated environments are implemented in MuJoCo~\citep{todorov2012mujoco} using LIBERO~\citep{liu2023libero}:
\begin{enumerate}
    \item \emph{Cleanup (simulated).} A Franka Panda robot must clear a table by placing an object into one of three drawers. The task MDP has $47$ abstract states and $10$ skills in total, with the longest task plan consisting of $4$ skills. The robot is provided abstract specifications of \texttt{Pick}, \texttt{Place}, \texttt{OpenDrawer}, and \texttt{CloseDrawer} skills.
    \item \emph{Cleanup-Multi (simulated).} This is a challenging multi-object version of the \emph{Cleanup} task in which four objects must be placed into a drawer. The task MDP has $5000$ abstract states and $22$ skills in total, with the longest task plan consisting of $10$ skills.
    \item \emph{Breakfast (real-robot).} A Franka Panda robot can toast bread or make hot oatmeal using a microwave to achieve a higher reward. 
    The former requires learning a \texttt{StartToaster} skill, while the latter requires learning \texttt{OpenMicrowave} and \texttt{CloseMicrowave} skills, and hence is harder. These skills require forceful interaction with novel articulated objects and hence must be learned by practicing. The robot has additional pick and place skills for transporting objects using motion planning that need not be learned.
\end{enumerate}

\begin{figure}[t]
    \centering
    \includegraphics[width=1\linewidth]{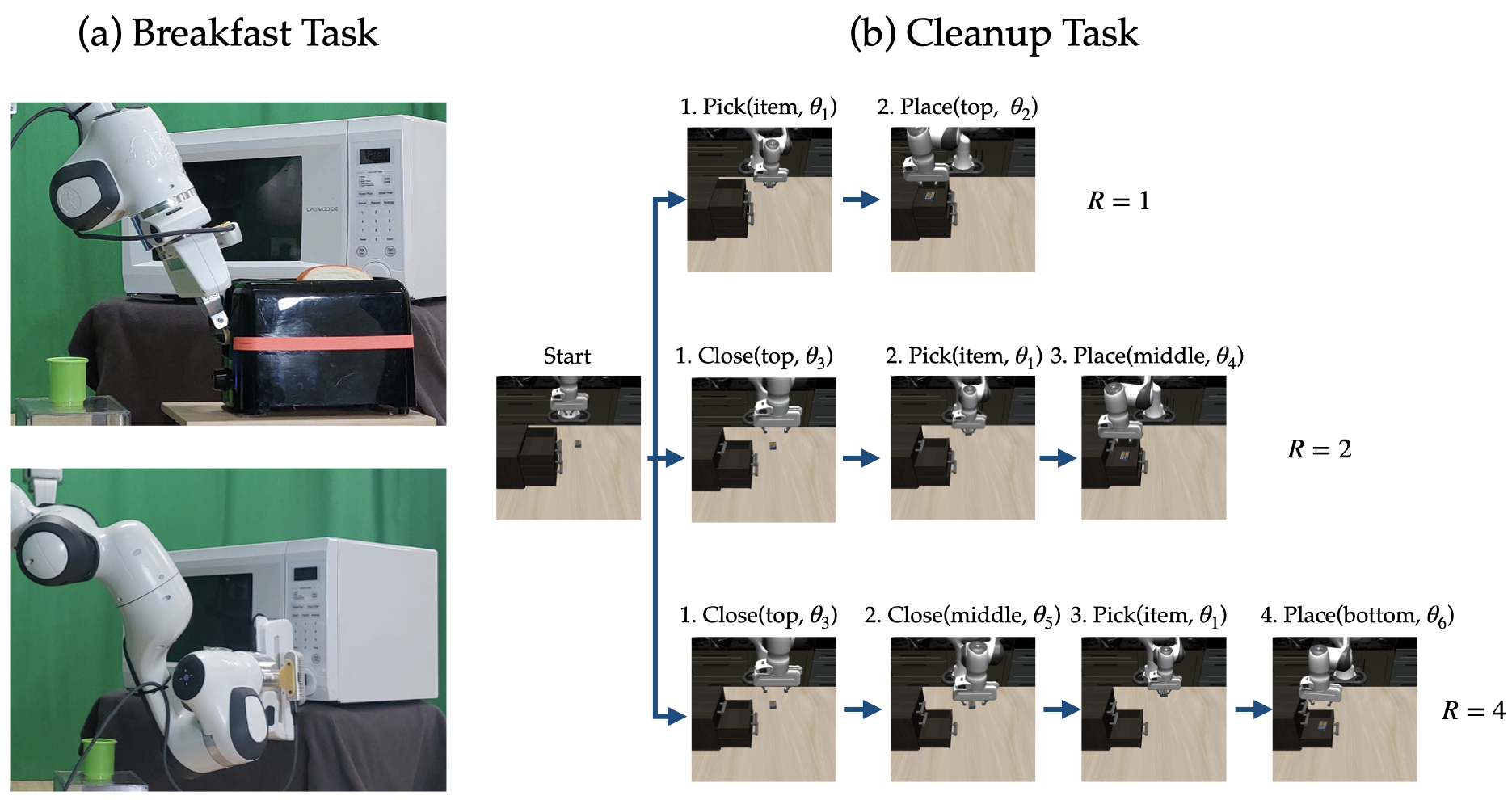}
    \caption{(a) \emph{Breakfast task} requires learning forceful manipulation skills to interact with novel articulated objects.  (b) \emph{Cleanup task} requires the robot to place objects in the {top}, {middle}, or {bottom} drawer with rewards $1$, $2$ and $4$ respectively.
    }
    \label{fig:sim_cleanup}
\end{figure}

\textbf{Baselines.} We compare \emph{Deliberate Practice (DP)} against the following baselines:
    (1) \emph{Estimate, Extrapolate, Situate (EES)} estimates the expected improvement of each skill and greedily practices the skill that most improves task performance~\citep{kumar2024practice}.
    (2) \emph{Competence Improvement (CI)} practices the skill with the highest expected competence improvement~\citep{stout2010competence,colas2019curious}.
    (3) \emph{Least Competent First (LCF)} practices the skill with the lowest estimated competence.
    (4) \emph{Random (R)} randomly samples a reachable skill for practice.
Implementation details are provided in the Appendix.

\textbf{Skill Practice.} We implement all learnable skills as Cartesian-space impedance controllers~\cite{hogan1985impedance} parameterized by waypoints and impedance parameters.
For each skill, we define a reward function and optimize the corresponding interaction policy parameters using CMA-ES~\cite{hansen2001completely}. We use CMA-ES because our objective is evaluated through rollouts and is not differentiable with respect to the policy parameters. CMA-ES is a robust gradient-free optimizer for such continuous black-box objectives.

\subsection{Experimental Results}

\begin{figure*}[t]
    \centering
    \includegraphics[width=1\linewidth]{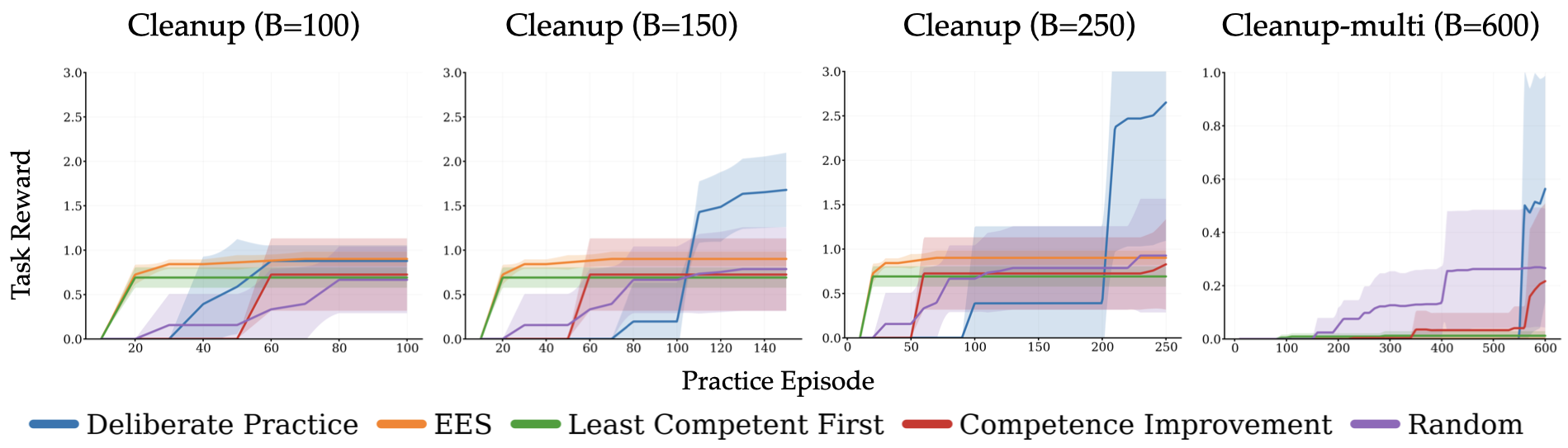}
    \caption{We evaluate all methods on \emph{Cleanup} with budgets (B) of 100 (low), $150$ (medium), and $250$ (high) episodes. While they perform similarly under a low budget, DP is significantly better under medium and high budgets. We report the mean and standard deviation averaged over $5$ seeds.}
    \label{fig:learning_curves}
\end{figure*}

\textbf{DP is budget-aware.} \Figref{fig:sim_cleanup} shows three possible task plans in the \emph{Cleanup} task that place items in the top, middle and bottom drawers to achieve rewards of 1, 2, and 4, respectively.
We compare the effect of different practice budgets on all methods in \Figref{fig:learning_curves}.
DP intelligently chooses which skills to learn based on the available budget: conservatively placing items in the top drawer when the budget is 100 episodes, placing items in the middle drawer under a budget of 150 episodes, and maximizing reward by placing items in the bottom drawer under  a budget of 250 episodes.
By contrast, baselines cannot adapt their behavior to the budget and greedily practice the easiest (and lowest-reward) task plan, irrespective of the budget.

\textbf{DP significantly outperforms prior active learning methods.}
As shown in \Figref{fig:learning_curves}, DP and EES perform similarly under a low budget, but DP significantly outperforms all baselines in medium (150 episodes) and high budget (250 episodes) settings.
The greedy active learning baselines, EES and CI, reason myopically and, therefore, fail to discover higher-reward plans which require practicing skills that do not immediately improve task performance.
This trend continues in the more complex \emph{Cleanup-Multi} task, which requires sequencing 10 skills.
Interestingly, \emph{random} practice outperforms EES in this setting because EES lacks an exploration mechanism when the myopic task improvement measure $\Delta J_{\text{task}}$ is zero for all skills.
By contrast, DP can look ahead and hence does not suffer from this limitation.

\textbf{DP for larger problems.}
We evaluate the scalability of  DP in \emph{Cleanup-Multi}, which has $22$ skills and $5000$ abstract states resulting in a large bilinear program.
Despite the size of this problem, DP consistently computes the optimal budget allocation within a maximum solve time of $6$ minutes.
In larger settings where the solver may not prove global optimality within the allotted time, the same optimization procedure can still return a bounded-suboptimal allocation as Gurobi maintains an incumbent feasible solution and a global bound from convex relaxations of the bilinear constraints, providing an optimality-gap certificate upon early termination.
\begin{wrapfigure}{r}{0.4\columnwidth} %
    \centering
    \includegraphics[width=\linewidth]{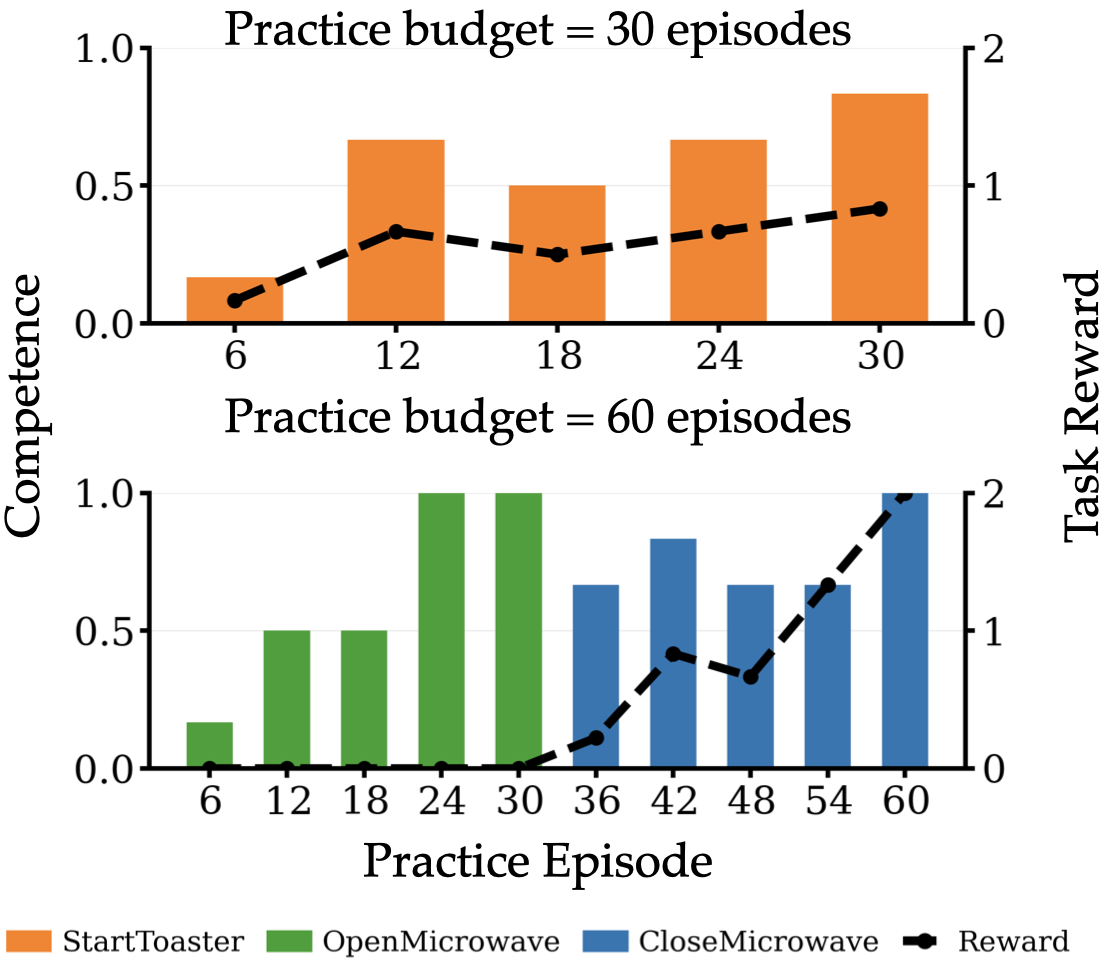}
    \caption{Real-world \emph{Breakfast} task.}
    \label{fig:breakfast}
\end{wrapfigure}

\textbf{Real-world Validation.}
We conduct experiments on a real Franka Panda robot to validate our design choices.
To complete the \emph{Breakfast} task, the robot must learn either to toast bread or microwave oatmeal.
As shown in \Figref{fig:teaser}, toasting provides a reward of 1 and requires practicing only the \texttt{StartToaster} skill, while microwaving oatmeal provides a reward of 2, but requires practicing two skills: \texttt{OpenMicrowave} and \texttt{CloseMicrowave}. The robot is provided with additional \texttt{Pick} and \texttt{Place} skills that can be used to transport items.
We evaluate DP under two practice budgets: 30 episodes and 60 episodes. As shown in ~\Figref{fig:breakfast}, DP correctly chooses the conservative toast-bread plan under the smaller budget, but learns to microwave oatmeal under the larger budget, thereby achieving higher task reward.

\section{Conclusion}
We propose \methodname{}, an active skill learning algorithm that allocates a limited practice budget to many robot skills to provably maximize long-horizon planning performance.
Our algorithm models each skill's competence as a function of the practice budget, computes an optimal budget allocation, and sequentially practices the skills in the environment.
Computing this allocation requires solving a challenging bilevel optimization problem, for which we derive an exact single-level reformulation that can be solved using off-the-shelf solvers.
Simulated and real-robot experiments show that our method is uniquely budget-aware and allows robots to better leverage their practice budgets compared to existing methods to master long-horizon manipulation tasks.

\section{Limitations}
One key assumption of our approach is access to approximate priors over skill competence. If these priors are overly optimistic, the robot may allocate practice to task plans that are actually infeasible within the available budget. This limitation could be addressed by using conservative priors or by explicitly incorporating uncertainty in skill-competence predictions.
A second limitation is that solving the bilinear program to global optimality may become challenging for very large problems. An important direction for future work is therefore to develop bounded-suboptimal optimization strategies.
Finally, we plan to extend our approach to practicing generalizable skills acros multiple tasks on mobile manipulators, moving toward the broader goal of general-purpose robots.

\acknowledgments{
This work was supported by the Office of Naval Research (ONR) under REPRISM MURI N000142412603 and ONR grant N00014-22-1-2592, as well as by the National Science Foundation (NSF) via grant 1955361 and ARL grant W911NF-18-2-0218. Partial funding was also provided by the Robotics and AI Institute.
Sudarshan was supported by the Office of Naval Research (ONR) grant N00014-22-1-259.
}

\clearpage

\bibliography{references}

\clearpage

\appendix
\section{Appendix}
\subsection{Linear Programming Formulation of MDPs}
We leverage the linear programming (LP) formulation of MDPs to instantiate \texttt{SolveMDP} in \Eqref{eq:bilevel_obj}.
Using the primal LP formulation for computing the optimal value function of an MDP, we get the following bilevel problem:
\begin{align}
\begin{split}
    &\max_b \sum_{s\in\Sa} e_s v_s \\
    &\text{subject to} \\
    &\quad \sum_{u \in \mathcal{U}} b_{u} \leq B \\ 
    &\quad v_s \in \argmin_{v} \sum_{s} e_s v_s \\
    &\quad\text{subject to } \\
    &\quad\quad v_s \geq r_s^a + \gamma \sum_{s'}\bar{P}_{ss'}^a(b) v_{s'},\; \forall a, s  \\
    &\quad\quad v_s \text{ unconstrained},\; \forall s
\end{split}
\label{eq:bilevel2}
\end{align}
This is a bilevel program (also called a Stackelberg game), where the inner program (follower) is conditioned on decision variables that are chosen by the outer program (leader).
In particular, the follower's constraints depend on the budget decided by the leader since the transition function $\bar{P}$ is a function of the budget $b$ via the competence improvement function $\bar{P}_{ss'}^a = f_{\text{improve}}(P_{ss'}^a, b)$

\subsection{Single-Level Reformulation} \label{sec:single_level_formulation}
The follower is a $b$-parameterized LP for which Slater's conditions hold. 

\begin{lemma}\label{lemma:slaters}
Suppose $b$ is given. Then the LP $\min_{v} \{e^Tv: v_s \geq r_s^a + \gamma \sum_{s'}\bar{P}_{ss'}^a(b) v_{s'},\; \forall s, a\}$ satisfies Slater's condition. 
\end{lemma}
\begin{proof}
The satisfaction of Slater's condition is equivalent to showing that there exists $\hat{v}$ such that inequalities are strictly satisfied. Let $\alpha = \max_{s,a} r^a_s$. Then, it is easy to verify $\hat{v}_s = \frac{\alpha}{\gamma^2} \,\forall\, s$ satisfies the inequalities strictly. This proves the claim.
\end{proof}
It is intractable to directly solve the bilevel program.
A standard approach for reformulating such bilevel problems into single-level problems is to replace the follower with its KKT conditions since they are both necessary and sufficient for optimality. 
This results in a single-level optimization with complementarity constraints which can be solved using mixed integer programming. However, this formulation is usually difficult to solve in practice.
Our key observation is that \emph{both the leader and the follower are optimizing the same objective}. Hence, the bilevel problem can be rewritten as a $\max$-$\min$ problem:
\begin{align} 
\begin{split}
&\max_{b \in \mathcal{B}} \min_{v} \{e^Tv: v_s \geq r_s^a + \gamma \sum_{s'}\bar{P}_{ss'}^a(b) v_{s'},\; \forall s, a\} \\
&\text{where,} \\
&\quad \mathcal{B} := \{b_{u} : \sum{b_{u}} \leq B \}
\end{split}
\end{align}
Since the LP satisfies Slater's condition (Lemma~\ref{lemma:slaters}) \emph{strong duality} holds, i.e. the primal and dual are both feasible and attain their respective optimal values which are equal. 
By \textbf{strong duality}, the dual of the follower attains the same optimum as the primal. Hence, we can replace the inner $\min$ problem with its dual $\max$ LP.
\begin{align} 
\begin{split}
&\max_{b \in \mathcal{B}} \max_{\mu} \sum_{\forall s\in\Sa} \sum_{\forall a \in \Aa} r_s^a\mu_s^a \\
&\quad\quad\quad\text{subject to}\\
&\quad\quad\quad\quad\sum_{\forall a \in \Aa} \mu_s^a - \gamma \sum_{s'} \sum_{\forall a } P_{s's}^a \mu_{s'}^a = e_s, \forall s \\
&\text{where, }\mathcal{B} := \{b_{u} : \sum{b_{u}} \leq B \}
\end{split}
\end{align}
The inner problem is feasible for every choice of budget. Hence, we can jointly maximize the budget and dual variables in single $\max$ optimization: 
\begin{align} 
\begin{split}
&\max_{b, \mu} \sum_{\forall s\in\Sa} \sum_{\forall a \in \Aa} r_s^a\mu_s^a \\
&\text{subject to}\\
&\quad\sum_{\forall a \in \Aa} \mu_s^a - \gamma \sum_{s'} \sum_{\forall a } \bar{P}_{s's}^a(b) \mu_{s'}^a = e_s, \forall s \\
&\quad\sum{b_{u}} \leq B
\end{split}
\end{align}
This is a surprisingly simpler single-level program, where the budget constraints have been added to the dual LP formulation of the MDP. Note that it is a non-linear problem even when the competence improvement function is linear.
However, standard optimizers, such as Gurobi~\citep{gurobi} can natively handle such bilinear constraints via the classic McCormick Envelope~\citep{mccormick1976computability}.

\subsection{Experimental Details}
\subsubsection{Skill Practice}
We use CMA-ES to practice robot skills.
CMA-ES is a gradient-free evolutionary-based optimization approach for reward maximization. 
For each skill, we define a reward function and optimize the corresponding interaction policy parameters using CMA-ES~\cite{hansen2001completely}. We use CMA-ES because our objective is evaluated through rollouts and is not differentiable with respect to the policy parameters. CMA-ES is a robust gradient-free optimizer for such continuous black-box objectives.
We implement all learnable skills as Cartesian-space impedance controllers~\cite{hogan1985impedance} parameterized by waypoints and impedance parameters.
The optimizer takes as input the policy parameters to be optimized, together with a Gaussian prior over these parameters, specified by a mean and diagonal covariance matrix.

At each iteration of CMA-ES, we sample $N = 6$ candidate parameter vectors from the current Gaussian search distribution. Each candidate is evaluated by executing the corresponding policy on the real robot in the environment. At the end of each episode, the robot receives a binary reward, and the top half candidates are used to update the mean and covariance of the search distribution. Over successive iterations, the search distribution shifts toward regions of the parameter space that provide higher rewards and improve competence.

\subsection{Baselines}
\begin{enumerate}
    \item \textbf{Estimate, Extrapolate, Situate (EES)}  estimates the expected improvement from practicing each skill independently using $f_{\text{improv}}(\pi, 1)$. For each skill, it uses a task planner to compute the expected improvement in task performance after practicing that skill for one round:
    \[\argmax_{\pi} \Delta J_{\text{task}}(\pi) = \argmax_{\pi} J_{\text{task}}(\bar{P}_{\pi}) - J_{\text{task}}(P),\]
    where $\bar{P}_{\pi}$ is the skill transition matrix after practicing $\pi$ for one round.
    $J_{\text{task}}$ is computed by solving the MDP induced by $P$ using its linear programming formulation.
    The robot greedily practices the skill predicted to yield the largest task-performance improvement in the next round~\citep{kumar2024practice}. 
    \item \textbf{Competence Improvement (CI)} practices the skill with the highest expected competence improvement~\citep{stout2010competence,colas2019curious}, i.e.,
    \[\argmax_{\pi} f_{\text{improv}}(\pi, 1)\]
    CI is biased towards skills that are easy to learn and, hence often wastes time on skills that are not relevant to the task.
    \item \textbf{Least Competent First (LCF)} practices the skill with the lowest current competence, i.e,. 
    \[\argmin_{\pi} p_{\pi}\]
    This encourages the robot to focus on underdeveloped skills and gradually master the full skill set. However, it can waste time on skills that are not relevant to the task or are too hard to learn within the practice budget.
    \item \textbf{Random (R)} unformly randomly samples a reachable skill for practice.
\end{enumerate}

\subsection{Real-Robot Breakfast Domain}
\begin{figure}[t]
    \centering
    \includegraphics[width=1\linewidth]{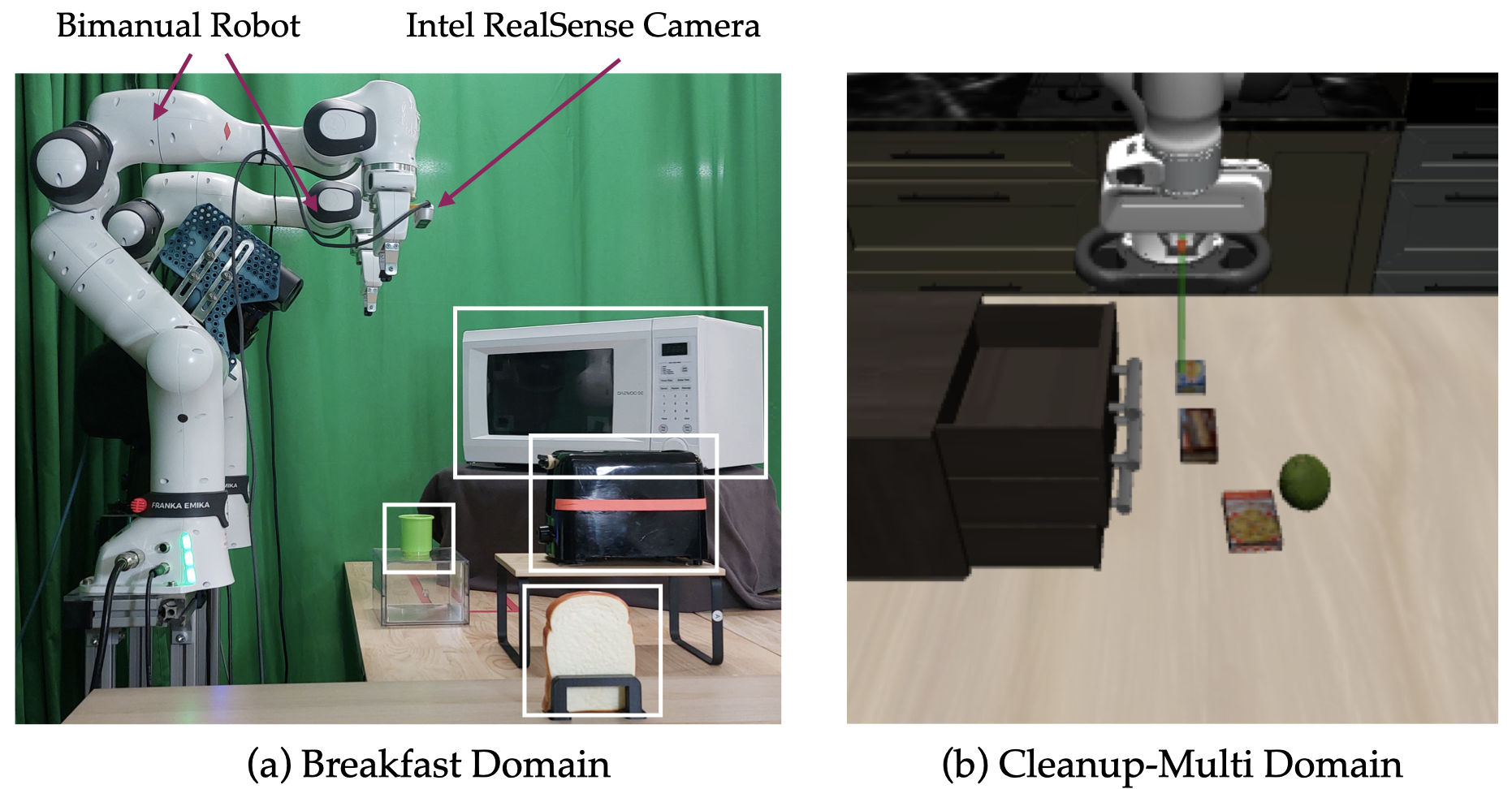}
    \caption{(a) \emph{Breakfast domain.} The robot can either toast bread to achieve reward $1$ or microwave oatmeal to achieve reward $2$. The former requires practicing \texttt{StartToaster} while the latter requires practicing \texttt{OpenMicrowave} and \texttt{CloseMicrowave}. (b) \emph{Cleanup-multi domain.} The achieves reward $1$ for placing all objects in the top drawer, reward $2$ for placing them in the middle drawer and reward $4$ for placing them in the bottom drawer. The lower drawers are blocked by the upper drawers, so the robot must learn \texttt{CloseDrawer} skills to place objects in the lower drawers.}
    \label{fig:breakfast_cleanup_domains}
\end{figure}

\begin{figure}[t]
    \centering
    \includegraphics[width=1\linewidth]{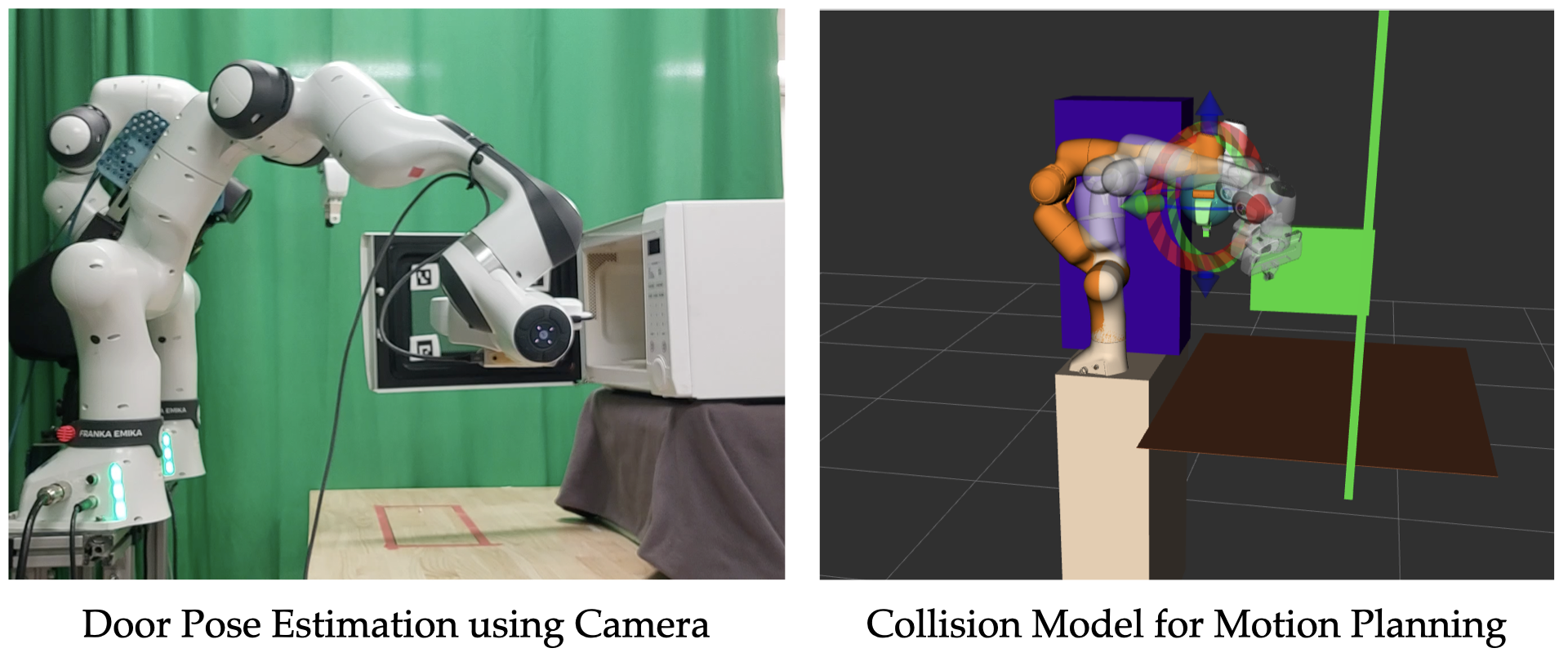}
    \caption{\emph{Breakfast domain.} After pressing the microwave button, the robot estimates the position of the microwave door using a RealSense wrist camera. This information is used to construct a collision model of the scene for motion planning.}
    \label{fig:door_pose}
\end{figure}
As shown in~\Figref{fig:breakfast_cleanup_domains}, the \emph{Breakfast} task is conducted on a bimanual robot with two Franka Panda arms..
The skill library comprises $7$ skills: \texttt{PickBread}, \texttt{PlaceBread}, \texttt{StartToaster}, \texttt{OpenMicrowave}, \texttt{PickBowl}, \texttt{PlaceBowl}, \texttt{CloseMicrowave}.
Our perception system includes a wrist mounted Intel-Realsense D435 camera.
Auroco Markers are mounted on the interior of the microwave door. Once the door is opened, the markers are detected using OpenCV Auroco Detection~\cite{garrido2014automatic,bradski2000opencv}, and depth measurements are used to estimate the dimensions and position of the microwave door. 
The resulting detected pose is used to construct a collision model of the scene, which enables collision-aware motion planning using OMPL~\cite{sucan2012open} through MoveIt~\cite{chitta2012moveit}.
Contact-rich skills \texttt{StartToaster}, \texttt{OpenMicrowave}, \texttt{CloseMicrowave} require forceful interaction with the toaster and microwave and are implemented using Cartesian-space impedance controllers parameterized by relative actions in the end-effector frame and impedance parameters that determine the forces applied by the robot during interaction. These impedance parameters are sensitive to the object properties and task context, and therefore must be learned through practice. To ensure safe exploration, we specify a bounded range of impedance parameters that the robot may explore and use a reward function to encourage the robot to apply the minimum force necessary to solve the task.

\subsubsection{Cleanup Domain}
The \emph{Cleanup} domain requires a Franka Panda robot to place an item in one of three open drawers to clean the table.
\emph{Cleanup-multi} is a harder version with $4$ items on the table (shown in~\Figref{fig:breakfast_cleanup_domains}).
Both of them are implemented in MuJoCo~\citep{todorov2012mujoco} using LIBERO~\citep{liu2023libero} which provides realistic physics simulation.
The robot receives higher reward for placing items in lower drawers. However, these placements are more challenging to learn because they require the robot to first learn how to close the upper drawers.
The \emph{Cleanup} domain has $41$ abstract  states, and \emph{Cleanup-multi} has $5001$ abstract states consisting of the following variables: \texttt{DrawerState} (open, closed) for every drawer, \texttt{ObjectState} (in-hand, on-table, in-top-drawer, in-middle-drawer, in-bottom-drawer) for every object, and a terminal failure state.
The robot is provided with abstract specifications for  $22$  manipulation skills including their preconditions, termination conditions and effects: 
\begin{enumerate}
\item \texttt{OpenDrawer} $\times$ \{top, middle, bottom\} drawer: opens a drawer by grasping the handle and pulling it. The pull action is learned through practice.
\item \texttt{CloseDrawer} $\times$ \{top, middle, bottom\} drawer: closes a drawer by pushing. The push action is learned through practice.
\item \texttt{Pick} $\times\ 4$ items: picks an object with a grasp pose learned through practice.
\item \texttt{Place} $\times\ 4$ items $\times\ 3$ drawers: moves the object to the target drawer and places it. The object release position depends on the object size and the target drawer and is learned from practice.
\end{enumerate}
All skills are implemented as Cartesian-space impedance controllers with a fixed impedance and waypoints defined with respect to the target object.
We use privileged information about object positions from the simulator to prameterize all object-centric skills.
The highest reward task plan requires sequencing $10$ skills to first close the top two drawers and place all items in the lowest drawer.

\end{document}